\documentclass[11pt]{article}
\usepackage{amsmath,amsthm}
\newtheorem{theorem}{Theorem}

\usepackage[]{acl}

\usepackage{times}
\usepackage{latexsym}
\usepackage[most]{tcolorbox}

\usepackage{listings}

\usepackage[T1]{fontenc}

\usepackage[utf8]{inputenc}

\usepackage{microtype}

\usepackage{inconsolata}

\usepackage{graphicx}

\usepackage{amsmath}
\usepackage{amsthm}
\usepackage{amsfonts}
\usepackage{amssymb}
\usepackage{mathtools}
\usepackage{nicefrac}
\usepackage{algorithm}
\usepackage{multicol}
\usepackage{algorithmicx}
\usepackage{algpseudocode}

\usepackage{graphicx}
\usepackage{subcaption}  
\usepackage{wrapfig}
\usepackage{booktabs}
\usepackage{array}
\usepackage{multirow}
\usepackage{adjustbox}
\usepackage{lscape}
\usepackage{caption}
\usepackage{wrapfig}
\usepackage{algpseudocode}
\usepackage{cleveref}

\definecolor{darkgreen}{rgb}{0.0, 0.4, 0.0}
\usepackage{threeparttable}

\usepackage{xcolor}
\usepackage{colortbl}
\usepackage{enumitem}
\usepackage{makecell}
\definecolor{colorTab}{rgb}{0.9,0.9,0.98}
\definecolor{color3}{gray}{0.95}
\usepackage[table]{xcolor}
\definecolor{codeblue}{rgb}{0.25, 0.5, 0.5}
\definecolor{codekw}{rgb}{0.35, 0.35, 0.75}
\definecolor{Gray}{gray}{0.95}
\definecolor{TC}{rgb}{0.75, 0.30, 0.25}  
\definecolor{SC}{rgb}{0.30, 0.50, 0.80}  

\usepackage[most]{tcolorbox}
\newtcolorbox{paperbox}[2][]{%
  enhanced,
  colback=white,
  colframe=black,
  fonttitle=\bfseries,
  before=\refstepcounter{table}, 
  title={#2},
  #1
}
\newtcolorbox{paperboxdouble}[2][]{%
  float*=t,              
  width=\textwidth,    
  enhanced,
  colback=white,
  colframe=black,
  fonttitle=\bfseries,
  code={\refstepcounter{table}}, 
  title={Table~\thetable: #2},
  #1
}

\title{Learning While Staying Curious: Entropy-Preserving Supervised Fine-Tuning via Adaptive Self-Distillation for Large Reasoning Models}

\author{
\textbf{Hao Wang}$^{1,*}$ \quad
\textbf{Hao Gu}$^{2,*}$ \quad
\textbf{Hongming Piao}$^{1}$ \quad
\textbf{Kaixiong Gong}$^{3}$\\
\textbf{Yuxiao Ye}$^{2}$ \quad
\textbf{Xiangyu Yue}$^{3}$ \quad
\textbf{Sirui Han}$^{2,\dagger}$ \quad
\textbf{Yike Guo}$^{2,\dagger}$ \quad
\textbf{Dapeng Wu}$^{1,\dagger}$\\
$^{1}$City University of Hong Kong \quad
$^{2}$The Hong Kong University of Science and Technology \quad \\
$^{3}$The Chinese University of Hong Kong
\\
hao.wang@my.cityu.edu.hk \quad siruihan@ust.hk \quad 	yikeguo@ust.hk \quad dapengwu@cityu.edu.hk}

\newcounter{noteHWctr} 

\newcommand{\oursolution}{CurioSFT}

\begin{document}

\maketitle
\begingroup
\renewcommand{\thefootnote}{} 
\footnotetext{\footnotesize $^{*}$Equal contribution \hspace{1.2em} $^{\dagger}$Corresponding author}
\endgroup
\setcounter{footnote}{0}
\begin{abstract}
The standard post-training recipe for large reasoning models, supervised fine-tuning followed by reinforcement learning (SFT-then-RL), may limit the benefits of the RL stage: while SFT imitates expert demonstrations, it often causes overconfidence and reduces generation diversity, leaving RL with a narrowed solution space to explore. Adding entropy regularization during SFT is not a cure-all; it tends to flatten token distributions toward uniformity, increasing entropy without improving meaningful exploration capability. In this paper, we propose \textbf{\oursolution}, an entropy-preserving SFT method designed to enhance exploration capabilities through intrinsic curiosity. It consists of (a) \textit{Self-Exploratory Distillation}, which distills the model toward a self-generated, temperature-scaled teacher to encourage exploration within its capability; and (b) \textit{Entropy-Guided Temperature Selection}, which adaptively adjusts distillation strength to mitigate knowledge forgetting by amplifying exploration at reasoning tokens while stabilizing factual tokens. Extensive experiments on mathematical reasoning tasks demonstrate that, \textit{in SFT stage}, \oursolution~outperforms the vanilla SFT by \textbf{2.5 points} on in-distribution tasks and \textbf{2.9 points} on out-of-distribution tasks. We also verify that exploration capabilities preserved during SFT successfully translate into concrete gains \textit{in RL stage}, yielding an average improvement of \textbf{5.0 points}. Code is available at https://anonymous.4open.science/r/CurioSFT.
\end{abstract}

\begin{figure}[t]
  \includegraphics[width=0.90\columnwidth]{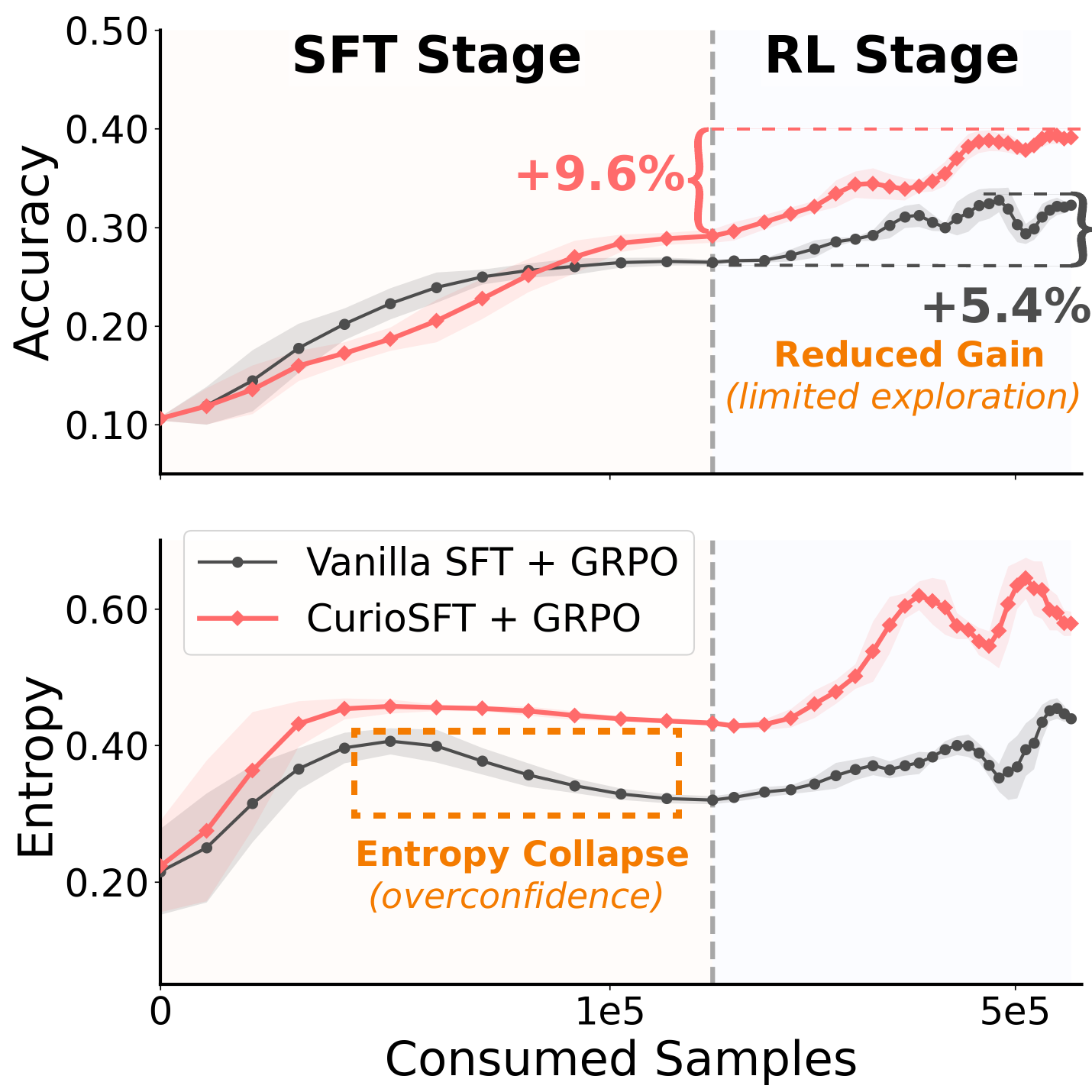}
  \caption{Evaluation entropy and accuracy (Avg@8 in AIME 2024) across the SFT and RL stages. \oursolution~mitigates entropy collapse during SFT, and yields larger accuracy gains in the RL stage.}
  \label{fig:intro_entropy}
  \vspace{-6mm}
\end{figure}

\section{Introduction}
Recent breakthroughs~\citep{gpt-oss,dpskr1} establish "SFT-then-RL" as the de-facto paradigm for enhancing large reasoning models on automatically verifiable tasks, such as mathematical reasoning~\citep{math_reasoning,kimi-k2}, code generation~\citep{code_reasoning}, and agentic search~\citep{dr_reasoning,search-r1}. In this paradigm, the Supervised Fine-tuning (SFT) stage aligns the model with domain-specific patterns and required knowledge, thereby providing a superior initialization for the subsequent Reinforcement Learning (RL) stage.

However, this paradigm faces a critical challenge: the Cross Entropy loss in SFT rigidly maximizes the likelihood of expert tokens, which inevitably drives the model toward overconfidence~\citep{overconf_1,overconf_2} and constricts the exploration space as training progresses. As shown in Figure~\ref{fig:intro_entropy}, we use token entropy to quantify exploration capability and observe a rapid collapse over the SFT stage. Counter-intuitively, the SFT stage locks the model into a low-diversity mode, severely constraining the search space for the subsequent RL stage. This limitation often leads to marginal gains or even degradation compared to direct RL, aligning with recent findings~\citep{sft_bad,sft_bad_extra}.

A straightforward approach is to regularize the SFT stage with entropy loss~\citep{jost2006entropy} on each token. However, trivially maximizing entropy will indiscriminately smooth the token probability and introduce \textit{ungrounded entropy}, damaging exploration capability and leading to unsatisfactory or degraded performance. Concretely, it fails to distinguish token roles: forcing entropy on factual tokens disrupts knowledge retention, while neglecting critical reasoning tokens (e.g., ``wait'') where exploration is truly beneficial. This discrepancy highlights the need for a method that \textit{substantially enhances exploration capabilities without compromising the model's intrinsic knowledge}.

To achieve this, we introduce \textbf{\oursolution}, a novel entropy-preserving SFT method designed to enhance exploration with knowledge retention. This method consists of two key components: \textit{Self-Exploratory Distillation} and \textit{Entropy-Guided Temperature Selection}. Building on self-distillation~\citep{sd_1,sd_2}, \textit{Self-Exploratory Distillation} exploits the monotonic relationship between token entropy and sampling temperature to construct a higher-entropy ``teacher distribution'' via an increased temperature. Aligning with this high-entropy teacher allows the model to selectively expand its search space under the guidance of its own curiosity. Crucially, to account for the distinct roles of tokens during reasoning, \textit{Entropy-Guided Temperature Selection} dynamically modulates the temperature based on token-level uncertainty. This mechanism selectively encourages exploration at critical reasoning tokens while maintaining deterministic targets for factual tokens, thereby effectively mitigating the risk of knowledge forgetting.

Extensive experiments on mathematical reasoning benchmarks demonstrate that, \oursolution~not only effectively preserves entropy but also achieves superior performance across both in-distribution and out-of-distribution (OOD) tasks, outperforming vanilla SFT by an average of \textbf{2.5 points} and \textbf{2.9 points}, respectively. We empirically verify that the exploration capabilities preserved during SFT successfully translate into concrete gains in the RL stage. To this end, our contributions are three-fold:
\begin{itemize}
    \item  We empirically analyze the drawbacks of entropy loss in SFT, including \textit{exploration degradation} and \textit{knowledge forgetting}. To address these, we propose \oursolution, which preserves entropy while improving overall performance during the SFT stage.
    
    \item We propose \textit{Self-Exploratory Distillation} to preserve entropy while improving effective exploration by aligning with a self-generated, temperature-scaled teacher. We further introduce \textit{Entropy-Guided Temperature Selection} to adapt token-level temperatures, selectively encouraging exploration and mitigating knowledge forgetting.
    
    \item Extensive experiments on mathematical reasoning benchmarks demonstrate that, \oursolution~not only improves performance in SFT but also enhances the exploration capability, significantly improving the performance of RL stage. We also verify the robustness of \oursolution~across models and hyperparameters.
\end{itemize}

\begin{figure*}[t]
  \centering

  \begin{subfigure}[b]{0.32\textwidth}
    \centering
    \includegraphics[width=\linewidth]{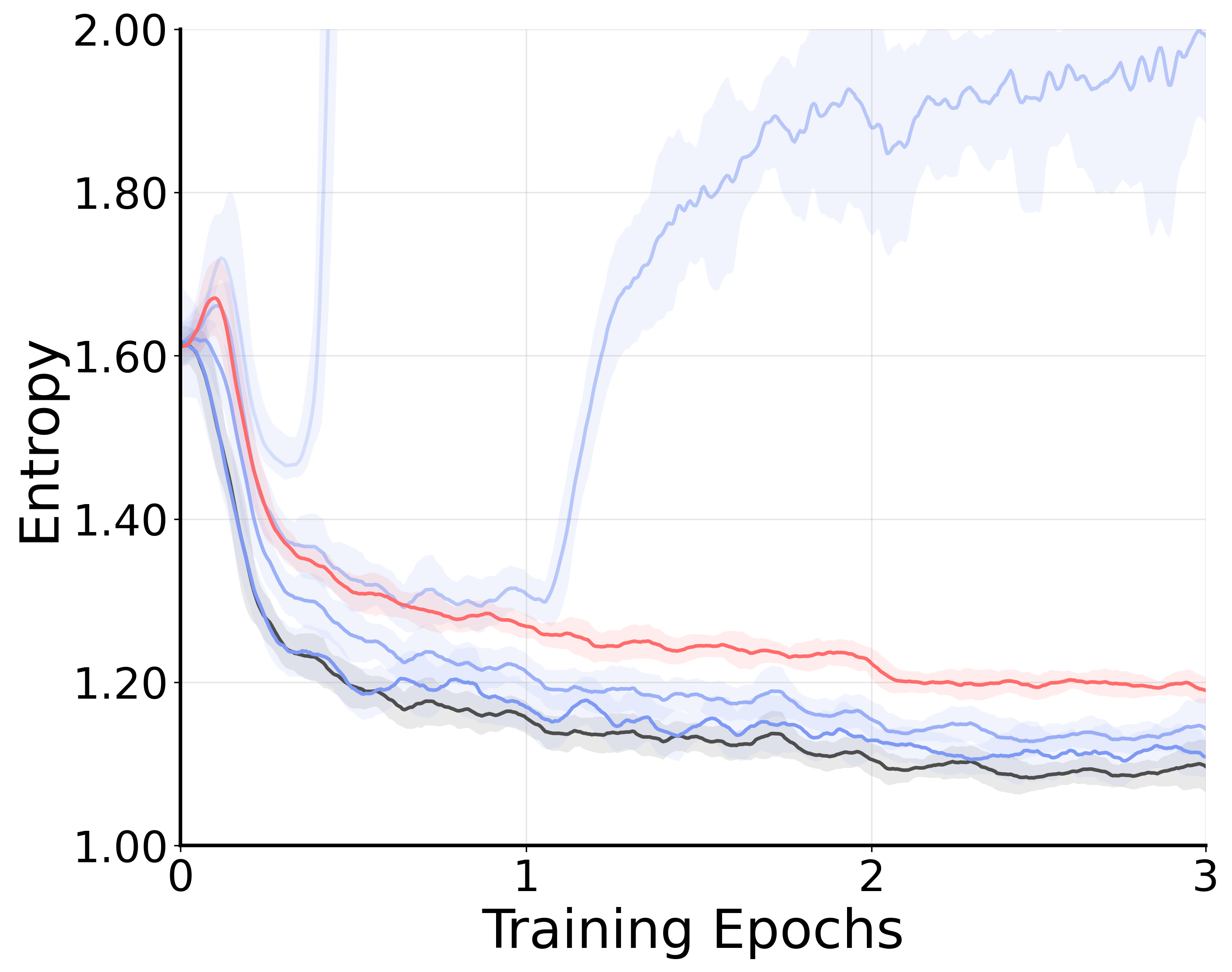}
    \vspace{-6mm}
    \caption{Expert Token Entropy (Top 20\%)}
    \label{fig:entropy-a}
  \end{subfigure}
  \begin{subfigure}[b]{0.32\textwidth}
    \centering
    \includegraphics[width=\linewidth]{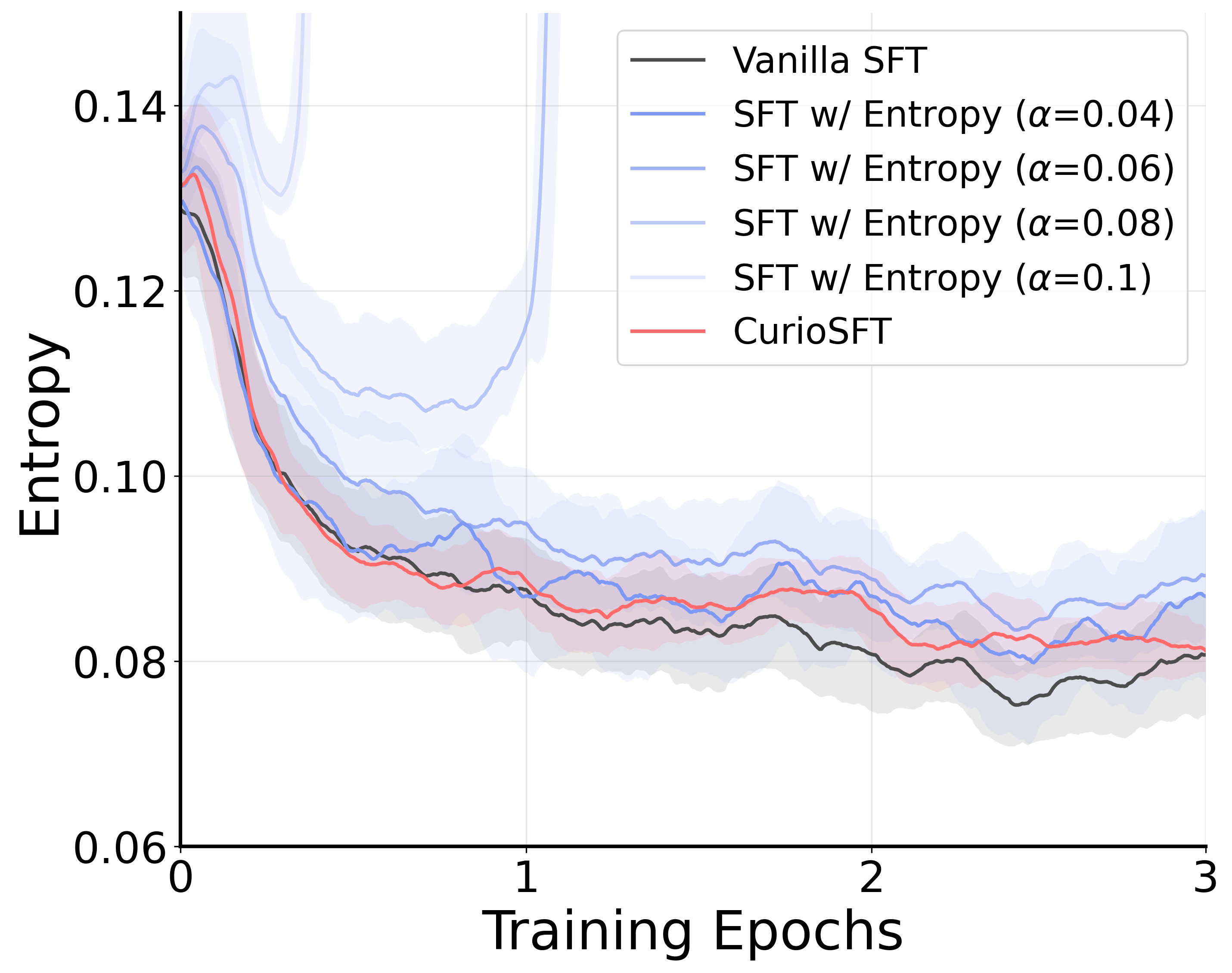}
    \vspace{-6mm}
    \caption{Expert Token Entropy (Bottom 80\%)}
    \label{fig:entropy-b}
  \end{subfigure}
  \begin{subfigure}[b]{0.32\textwidth}
    \centering
    \includegraphics[width=\linewidth]{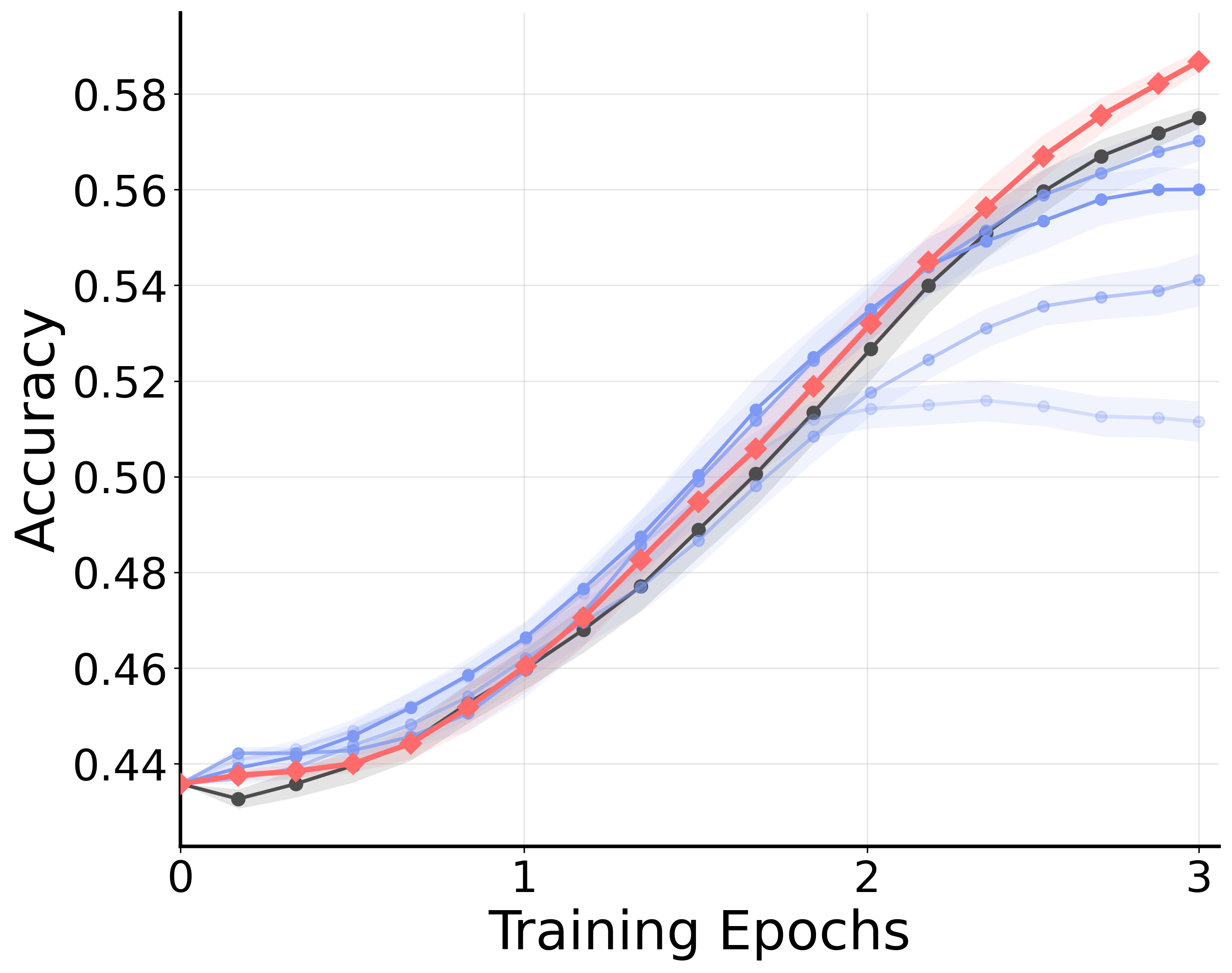}
    \vspace{-6mm}
    \caption{Pass@32 for AIME 2024}
    \label{fig:entropy-c}
  \end{subfigure}

  \vspace{-1mm}
\caption{\textbf{Expert token entropy and evaluation accuracy during the SFT stage.}
Compared to vanilla entropy loss, which uniformly encourages entropy across tokens,
\oursolution~selectively increases entropy on high-entropy tokens while preserving low-entropy ones.
We further observe that the increased token entropy induced by entropy loss does not translate into actual improvements in Pass@32 performance in our experiments.}
\vspace{-5mm}
  \label{fig:entropy-both}
\end{figure*}

\section{Preliminaries}
\paragraph{SFT Loss.} Let $\mathcal{D}$ denote the SFT dataset, which contains multiple questions ${\mathbf{x}}$ and corresponding expert responses $\mathbf{y}$. The optimization objective during the SFT stage is to minimize the cross-entropy loss between the model distribution and a one-hot target distribution induced by expert tokens, as:
\begin{align}
\mathcal{L}_{\text{SFT}}(\theta)=- \log \pi_\theta\big(y_t \mid \mathbf{s}_t\big),
\end{align}
where $\pi_\theta$ is the fine-tuned model and $\mathbf{s}_t$ is the concatenation of the question $\mathbf{x}$ and the previously generated tokens $\mathbf{y}_{<t}$. 

\paragraph{Entropy Loss.} 
To encourage output diversity and prevent over-confidence, prior works~\citep{math_reasoning,openreasonerzeor} introduce an entropy loss term as a regularizer during the RL stage, as:
\begin{equation}
\begin{aligned}
    \mathcal{L}_{\text{entropy}}(\theta) &= \alpha  \cdot -H(\pi_\theta(\cdot \mid \mathbf{s}_t)) \\
    &= \alpha \cdot \sum_{y \in \mathcal{V}} \pi_\theta(y \mid \mathbf{s}_t) \log \pi_\theta(y \mid \mathbf{s}_t),
\end{aligned}
\end{equation}
where $\mathcal{V}$ denotes the vocabulary of the fine-tuned model, $\alpha$ is the loss weight. However, applying entropy regularization solely at the RL stage often yields marginal benefits, as the preceding SFT stage has already driven the model into a low-entropy mode. Consequently, it is important to preserve entropy and encourage exploration during the SFT stage itself. Yet, deploying entropy loss in the SFT stage presents a fundamental challenge: \textit{unlike the online nature of RL, SFT is an offline process where the model cannot judge whether the increased entropy leads to valid reasoning paths or merely introduces noise}. In the following section, we discuss two key limitations arising from this ``blind'' regularization through empirical observations.

\section{The Pitfall of Entropy Loss in SFT}
\label{sec:entropy}

\begin{table}[t]
\centering
\caption{OOD performance comparison.}
\vspace{-2mm}
\label{tab:ood_entropy_top20}
\setlength{\tabcolsep}{6pt}
\renewcommand{\arraystretch}{1.1}
\resizebox{\columnwidth}{!}{
\begin{tabular}{lcccc}
\toprule
\textbf{Method} & \textbf{GPQA} & \textbf{MMLU-Pro} & \textbf{ARC-C} & \textbf{Avg.$\uparrow$} \\
\midrule
Vanilla SFT & 27.7 & 47.5 & 78.8 & 51.3 \\
SFT w/ Entropy & 28.0 & 45.9 & 77.2 & 50.4 \\
SFT w/ Entropy (Top 20\%) & \textbf{29.6} & \textbf{47.9} & \textbf{79.3} & \textbf{52.3} \\
\bottomrule
\end{tabular}
}
\vspace{-5mm}
\end{table}

\paragraph{Ungrounded entropy degrades exploration capability.}
In the SFT stage, many tokens in the dataset are relatively \emph{unfamiliar} to the current model, reflected by their lower output probability compared to online sampling tokens (offline 71\% vs.\ online 76\%). As a result, the entropy loss becomes highly sensitive to its weight $\alpha$: as shown in Figure~\ref{fig:entropy-a} and Figure~\ref{fig:entropy-b}, when $\alpha$ is too small, entropy barely increases; when $\alpha$ is too large (e.g., $\alpha \ge 0.08$ in our setting), the objective can push some token distributions toward near-uniformity, causing an ``entropy explosion'' that destabilizes training. Even with a seemingly reasonable choice (e.g., $\alpha=0.06$), entropy loss does not reliably improve performance (Figure~\ref{fig:entropy-c}). The key reason is that the entropy loss indiscriminately pushes the token distribution toward higher entropy, without distinguishing between expanding a valid reasoning path and merely injecting noise. Consequently, increased token entropy does not translate into better reasoning performance and may even harm effective exploration.

\paragraph{Token-agnostic regularization amplifies knowledge forgetting.}
Recent works suggest that exploration in LLMs is driven by a relatively small subset of high-entropy tokens, while most tokens remain low-entropy to preserve knowledge~\citep{80_20}. As shown in the Figure~\ref{fig:entropy-a} and Figure~\ref{fig:entropy-b}, when we partition tokens by entropy (e.g., top 20\% vs.\ the remaining 80\%), naive entropy loss increases entropy in \emph{both} groups. This is because maximizing entropy is equivalent to minimizing the KL divergence to a uniform distribution for \emph{all} tokens (detailed proof in Appendix~\ref{apd:entropy_derivation}). Such token-agnostic regularization is detrimental to the model's original knowledge and reasoning behavior. As shown in Table~\ref{tab:ood_entropy_top20},
restricting entropy loss to the top 20\% high-entropy tokens yields significantly better performance on knowledge-intensive OOD tasks. Empirically, low-entropy tokens often correspond to deterministic factual content (e.g., nouns and numbers), where stability is crucial; forcing entropy at these positions weakens factual consistency and can induce knowledge forgetting. In contrast, high-entropy tokens tend to act as reasoning connectors (e.g., ``wait'', ``alternatively''), which are the natural targets for exploration.

 \begin{figure*}[t]
  \centering 
  \includegraphics[width=0.88\textwidth]{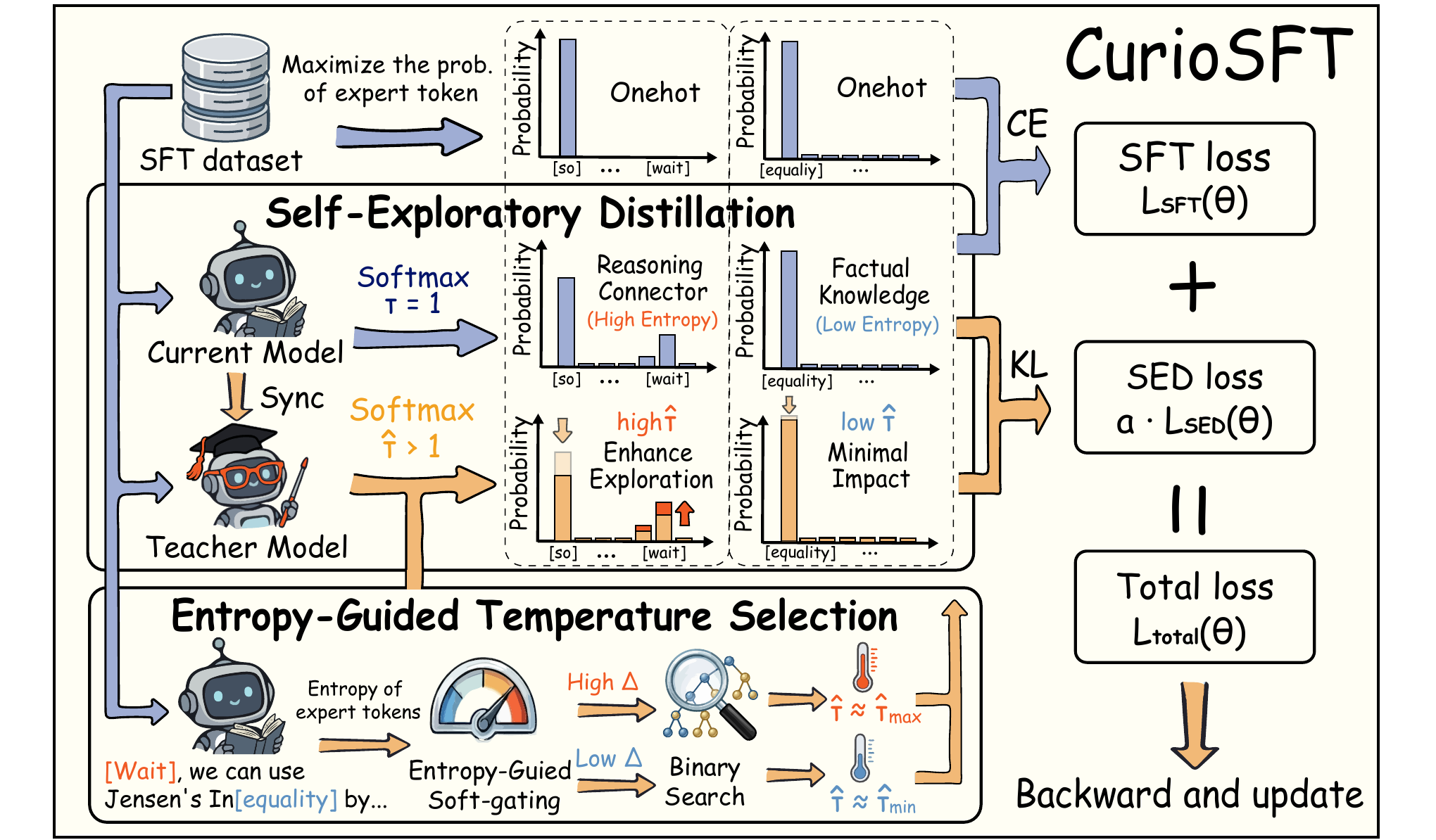}
  \vspace{-1mm}
  \caption{Proposed solution: \oursolution}
  \label{fig:solution} 
  \vspace{-5mm}
\end{figure*}

\section{Proposed Solution: \oursolution}
\label{sec:method}
To address the limitations of vanilla entropy loss, we introduce \oursolution, an entropy-preserving SFT method that enables models to learn expert behaviors while maintaining exploration capability. As shown in Figure~\ref{fig:solution}, our method consists of two key components: \textbf{Self-Exploratory Distillation} (Section~\ref{sec:sed}) and \textbf{Entropy-Guided Temperature Selection} (Section~\ref{sec:egat}).

\subsection{Self-Exploratory Distillation}
\label{sec:sed}
Frontier LLMs internalize extensive world knowledge during pre-training, and exhibit implicit exploration capabilities that enable the generation of diverse trajectories~\citep{implicit_prm1,implicit_prm2}. Motivated by self-distillation~\citep{solution_sd}, we exploit this intrinsic capacity by minimizing the divergence between the policy of current model and a self-generated, higher-entropy teacher distribution. Formally, an LLM policy is obtained by applying the softmax function to the output logits $z_\theta(\cdot\mid\mathbf{s}_t)$, as:
\begin{equation}
    \textcolor{SC}{\pi_\theta(y \mid \mathbf{s}_t; \textcolor{SC}{\tau})}=\frac{\exp\big(z_\theta(y \mid \mathbf{s}_t) / \textcolor{SC}{\tau}\big)}{\sum_{y' \in \mathcal{V}} \exp\big(z_\theta(y' \mid \mathbf{s}_t) / \textcolor{SC}{\tau}\big)},
\end{equation}
where $\textcolor{SC}{\tau}$ denotes the sampling temperature and is typically fixed at $1.0$ in standard SFT training. Leveraging the property that the entropy is monotonically increasing with respect to \textcolor{SC}{$\tau$} (see proof in Appendix~\ref{app:entropy_proof}), we can construct a higher-entropy teacher distribution $\pi^{\text{tch}}$ by simply rescaling the logits with a larger temperature $\textcolor{TC}{\hat{\tau}} > \textcolor{SC}{\tau}$, which satisfies $H\big(\textcolor{TC}{{\pi^{\text{tch}}}(\cdot \mid \mathbf{s}_t; \hat{\tau})}\big) > H\big(\textcolor{SC}{\pi_\theta(\cdot \mid \mathbf{s}_t; \tau)}\big)$. Subsequently, we can introduce a regularization loss that aligns the model with this higher-entropy teacher to achieve entropy preservation.

We theoretically prove that the constructed teacher distribution is the unique higher-entropy distribution that minimizes the KL divergence to the current policy under the entropy-increase constraint (see proof in Appendix~\ref{sec:t_scaling}). Adopting this teacher offers two key advantages: \textit{(a) Curiosity-driven exploration.} The temperature-scaled teacher strictly \textit{preserves the relative order} of token probabilities. Distilling the student toward this teacher therefore encourages exploration only over tokens that lie within the model’s \textit{valid exploration space}, rather than injecting uninformative entropy to all tokens. \textit{(b) Reduced Knowledge forgetting.} Prior works suggest that training data with lower divergence from the current model is associated with less knowledge forgetting~\citep{shenfeld2025rlsrazoronlinereinforcement}. By distilling toward a higher-entropy teacher that remains close to the current policy, the model can enhance its exploration ability while mitigating knowledge forgetting.

To ensure the stability of the teacher model, we deploy a separate teacher model parameterized by $\phi$, and the teacher distribution is denoted as:
\begin{equation}
    \textcolor{TC}{\pi_\phi^{\text{tch}}(y \mid \mathbf{s}_t;\hat{\tau})} = \frac{\exp(z^{\text{tch}}_\phi(y \mid \mathbf{s}_t) / \textcolor{TC}{\hat{\tau}})}{\sum_{y' \in \mathcal{V}} \exp(z_\phi^{\text{tch}}(y' \mid \mathbf{s}_t) / \textcolor{TC}{\hat{\tau}})},
\end{equation}
where $\textcolor{TC}{\hat{\tau}}$ is the teacher sampling temperature. Using a separate teacher updated more slowly than the student stabilizes the distillation target,
preventing rapid fluctuations in the teacher distribution during training. The parameters of the teacher model $\phi$ are synchronized with the current policy $\theta$ every $n$ steps using an exponential moving average with a decay factor of $\mu$ (Algo.~\ref{alg:main_training} Line 17). Finally, we formulate the self-exploratory distillation objective using the K2-loss~\citep{k3}, defined as:
\begin{equation}
\label{eq:loss_sd}
    \mathcal{L}_{\text{SED}}(\theta) = \frac{1}{2} \sum_{t=1}^{T} \left( \log \frac{\textcolor{SC}{\pi_\theta(y \mid \mathbf{s}_t;\tau)}}{\textcolor{TC}{\pi_\phi^{\text{tch}}(y \mid \mathbf{s}_t;\hat{\tau})}} \right)^2 .
\end{equation}
Finally, the overall optimization objective is defined as a combination
of the SFT loss and the self-distillation loss, with a coefficient $\alpha$:
\begin{equation}
\label{eq:loss_all}
    \mathcal{L}_{\text{total}}(\theta)
    = \mathcal{L}_{\text{SFT}}(\theta)
    + \alpha \cdot \mathcal{L}_{\text{SED}}(\theta).
\end{equation}

\subsection{Entropy-Guided Temperature Selection}
\label{sec:egat}

\begin{algorithm}[t]
\caption{Training with \oursolution}
\label{alg:main_training}
\begin{algorithmic}[1]
    \Require SFT dataset $\mathcal{D}$, base model $\pi_\theta$, teacher model $\pi_\phi^{\text{tch}}$
    \State Initialize teacher parameters: $\phi \leftarrow \theta$
    
    \For{$step = 1,2,\dots$}
        \State Sample training data $(\mathbf{x}, \mathbf{y}) \sim \mathcal{D}$
        \State Compute model logits $z_\theta(\cdot \mid \mathbf{s}_t)$
        \State \textcolor{gray}{// Entropy-guided temperature selection}
        \State Compute teacher logits $z_\phi^{\text{tch}}(\cdot \mid \mathbf{s}_t)$
        \State Compute entropy $H_t$ for each token $t$
        \State Compute entropy increment $\Delta_t$ by Eq.~(\ref{eq:entropy_increment})
        \State $\hat{\tau}_t \leftarrow \Call{BinarySearch}{z_\phi^{\text{tch}}(\cdot \mid \mathbf{s}_t), \Delta_t}$ 
        \State \textcolor{gray}{// Self-exploratory distillation}
        \State $\pi_\phi^{\text{tch}}(\cdot \mid \mathbf{s}_t) \leftarrow \text{Softmax}\!\big(z_\phi^{\text{tch}}(\cdot \mid \mathbf{s}_t) / \hat{\tau}_t\big)$
        
        \State Compute $\mathcal{L}_{\text{SED}}(\theta)$ by Eq.~(\ref{eq:loss_sd})
        \State Compute $\mathcal{L}_{\text{total}}(\theta)$ by Eq.~(\ref{eq:loss_all})
        \State Update $\theta \leftarrow \theta - \eta \nabla_\theta \mathcal{L}_{\text{total}}$
        
        \State \textcolor{gray}{// Teacher update}
        \If{$step\pmod{n} = 0$}
            \State $\phi \leftarrow (1-\mu) \phi +  \mu\theta$
        \EndIf
    \EndFor
\end{algorithmic}
\end{algorithm}

The sampling temperature $\textcolor{TC}{\hat{\tau}}$ for the teacher distribution is a crucial parameter in \oursolution. A higher value encourages the model to align with a higher-entropy teacher distribution, whereas a lower value keeps the update closer to standard SFT. As discussed in Section~\ref{sec:entropy}, \textit{high-entropy tokens} typically act as branching points that benefit from exploration, while \textit{low-entropy tokens} encode deterministic facts, where stability is preferred. To respect such heterogeneity, we adaptively assign temperatures based on the token uncertainty. We first compute a token-level entropy increment $\Delta_t$, then determine the temperature required via binary search. Finally, we use these specialized temperatures to construct the teacher distribution in Eq.~(\ref{eq:loss_sd}).

Specifically, given the current token entropy $H_t = H(\pi_\phi^{\text{tch}}(\cdot \mid \mathbf{s}_t))$, we compute the entropy increment $\Delta_t$ via:
\begin{equation}
\label{eq:entropy_increment}
    \Delta_t = \Delta_{\max} \cdot \text{Sigmoid}\big( \gamma \cdot (H_t - H_{\text{pivot}}) \big),
\end{equation}
where $\Delta_{\max}$ is the maximum allowable entropy increase, $\gamma$ is a scaling factor, and $H_{\text{pivot}}$ is the \textit{entropy pivot} that decides the activation threshold of exploration: increasing $H_{\text{pivot}}$ makes exploration more selective, while decreasing it expands the set of tokens receiving substantial entropy increase. We adopt soft-gating rather than a hard mask to avoid brittle thresholding and introduce a smooth, adaptive margin: for \textit{high-entropy tokens} ($H_t \gg H_{\text{pivot}}$), the sigmoid term approaches $1$, pushing the target entropy toward $H_t + \Delta_{\max}$ and thus strongly encouraging diversity at those positions. Conversely, for \textit{low-entropy tokens} ($H_t \ll H_{\text{pivot}}$), the sigmoid term approaches $0$, keeping the target entropy close to $H_t$ and thereby minimizing interference with the model's established knowledge.

Given the entropy increment $\Delta_t$, our goal is to find a  temperature for teacher distribution that matches the desired entropy target, as:
\[
\min_{\textcolor{TC}{\hat{\tau}_t}} \left| H\big(\pi_{\phi}^{\text{tch}}(\cdot \mid \mathbf{s}_t; \textcolor{TC}{\hat{\tau}_t})\big)
- \big( H_t + \Delta_t \big) \right| < \epsilon,
\]
where $\epsilon$ is a small constant. Given that entropy is monotonically increasing with respect to temperature $\tau$ (see proof in Appendix~\ref{app:entropy_proof}), we can efficiently solve for $\textcolor{TC}{\hat{\tau}_t} \in [\hat{\tau}_{\min}, \hat{\tau}_{\max}]$ using a binary search. To minimize computational overhead during training, we implement the temperature search as a fully vectorized operation and approximate the entropy using only the top-$k$ logits, which focuses computation on the most influential tokens while significantly accelerating the calculation. Details of binary search are provided in Appendix~\ref{app:binary_search}.

\section{Experiments}

\begin{table*}[t]
\centering
\caption{\textbf{Performance comparison in the SFT stage.}}
\label{tab:sft_results}
\vspace{-2mm}
\setlength{\tabcolsep}{3.5pt}
\renewcommand{\arraystretch}{1.08}
\resizebox{\textwidth}{!}{%
\begin{tabular}{lccccc>{\columncolor{orange!10}}c|ccc>{\columncolor{blue!10}}c|cc}
\toprule
\multirow{2}{*}{\textbf{Method}} &
\multicolumn{6}{c}{\textbf{In-Distribution Tasks}} &
\multicolumn{4}{c}{\textbf{Out-of-Distribution Tasks}} &
\multicolumn{2}{c}{\textbf{Other}} \\
\cmidrule(lr){2-7}\cmidrule(lr){8-11}\cmidrule(lr){12-13}
& \textbf{AIME25/24} & \textbf{AMC23} & \textbf{MATH.} & \textbf{Miner.} & \textbf{Olymp.} & \textbf{Avg.$\uparrow$}
& \textbf{GPQA} & \textbf{MMLU.} & \textbf{ARC-C}  & \textbf{Avg.$\uparrow$}
& \textbf{Entropy$\uparrow$} & \textbf{Speed$\downarrow$} \\
\midrule

\multicolumn{13}{c}{\textbf{Base Model}} \\
Qwen2.5-Math-7B  & 4.6/8.3 & 35.5 & 50.1 & 12.1 & 16.5 & 21.2 & 26.2 & 32.4 & 63.2 & 40.6 & 0.15  & --      \\
\midrule

\multicolumn{13}{c}{\textbf{Vanilla SFT with Regularization}} \\
Vanilla SFT      & 22.9/26.7 & 59.6 & 85.8 & 45.5 & 50.4 & 48.5  & 27.7& 47.5 & 78.8 & 51.3 & 0.31& \textbf{43.2}  \\
SFT with Entropy & 23.3/25.4 & 60.3 & 86.1 & 44.2 & 49.2 & 48.1 & 28.0& 45.9 & 77.2  & 50.4 & 0.36& 44.6  \\
SFT with KL      & 21.6/24.6 & 58.2 & 83.9 & 45.2 & 46.3 & 46.6 & 27.4& 47.8 & 78.0  & 51.1 & 0.30 & 50.2 \\
\midrule

\multicolumn{13}{c}{\textbf{SFT Variants}} \\
GEM~\citep{gem}  & 24.6/26.7 & 60.2 & 85.7 & 47.7 & 50.9 & 49.3 & 30.6 & 49.0 & \textbf{81.4}  & 53.7 & \textbf{0.66}& 44.5  \\
DFT~\citep{dft}  & 23.3/25.0 & 59.3 & 86.6 & 46.4 & 49.9 & 48.4 & 31.1 & 48.8 & 79.0 & 53.0 & 0.29 & 43.7 \\
PSFT~\citep{psft}& 25.0/28.8 & \textbf{60.4} & 86.9 & 48.3 & 52.6 & 50.3 & 29.9& 47.8 & 76.7  & 51.5 & 0.32& 45.2  \\
\midrule

\multicolumn{13}{c}{\textbf{Our Method}} \\
\oursolution & \textbf{26.3/29.6} & 59.9 & \textbf{87.0} & \textbf{49.8} & \textbf{53.2} & \textbf{51.0}
            & \textbf{31.7}  & \textbf{49.5} & 81.3 & \textbf{54.2}
            & 0.43 & 53.9  \\
\textit{Impr. vs SFT} & \textcolor{darkgreen}{\textbf{+3.4}}/\textcolor{darkgreen}{\textbf{+2.9}} & \textcolor{darkgreen}{\textbf{+0.3}} & \textcolor{darkgreen}{\textbf{+1.2}} & \textcolor{darkgreen}{\textbf{+4.3}} & \textcolor{darkgreen}{\textbf{+2.8}} & \textcolor{darkgreen}{\textbf{+2.5}} & \textcolor{darkgreen}{\textbf{+4.0}}& \textcolor{darkgreen}{\textbf{+2.0}} & \textcolor{darkgreen}{\textbf{+2.5}}  & \textcolor{darkgreen}{\textbf{+2.9}} & \textcolor{darkgreen}{\textbf{+0.12}} & \textcolor{black}{+10.7}  \\

\midrule
\textit{w/o Adaptive Temp.} & 24.6/26.7 & 59.0 & 86.4 & 48.8 & 53.1 & 49.8  & 29.5 &47.3  & 79.8&   52.2 & 0.45 & 51.9  \\
\textit{w/o Separate Teacher}     & 25.0/28.8 & 58.2 & 85.4 & 49.3 & 52.8 & 49.9 & 30.8  & 48.2 &80.2 & 53.0  & 0.39 & 45.5 \\
\bottomrule
\end{tabular}%
}
\vspace{-3mm}
\end{table*}

\paragraph{Training.} 
We use OpenR1-Math-46K~\citep{luffy} as the training dataset for both SFT and RL stages, which contains 46K mathematics problems and corresponding answers generated by DeepSeek-R1~\citep{dpskr1}. We adopt Qwen2.5-Math-7B~\citep{qwen2.5} as the base model, except in Section~\ref {sec:robustness}, where we study robustness across different models. For the SFT stage, we train for 3 epochs, and for the RL stage, we train for 500 steps using GRPO~\citep{grpo}. We set the entropy pivot $H_{\text{pivot}} = 1.2$ nats, the scaling factor $\gamma = 2.0$, the maximum entropy increment $\Delta_{\max} = 0.5$ nats, and the loss weight $\alpha = 1$. The temperature clip range is $\hat{\tau}_{\min}=1.1,\hat{\tau}_{\max}=1.5$. Further training details and hyperparameters are provided in Appendix~\ref{appendix-exp-setting}.

\paragraph{Evaluation.}
To evaluate the model's performance, we utilize six challenging and widely used mathematical reasoning benchmarks, including: AIME 2024, AIME 2025, AMC~\citep{numinamath}, Math-500~\citep{math500}, Olympiad Bench~\citep{olympiadbench}, and Minerva~\citep{minerva}. Furthermore, to assess the extent of knowledge retention and generalization ability, we evaluate the model on three OOD benchmarks: ARC-Challenge~\citep{arc_c}, GPQA-Diamond~\citep{gpqa}, and MMLU-Pro~\citep{mmlu_pro}. We set the sampling temperature to $0.6$ and $\text{Top\_p}=0.95$, and keep other settings consistent with the training. Due to the large number of questions in MMLU-Pro, we generate a single response per question for MMLU-Pro, while using $8$ responses per question for all other benchmarks, and compute the average accuracy as the final reported metric.

\subsection{Effectiveness of \oursolution}
\label{sec:sft_results}
\paragraph{Baselines.} We compare \oursolution~against two categories of baselines: \textit{(a) Vanilla SFT with Regularization}, which includes adding entropy loss and KL divergence constraints relative to the base model. \textit{(b) Advanced SFT Variants,} which are designed to mitigate over-confidence and encourage diversity. Specifically, PSFT~\citep{psft} employs trust-region constraints to limit policy shift; DFT~\citep{dft} re-weights token updates based on model internal knowledge; and GEM~\citep{gem} maintains diversity by encouraging the model to diverge from over-confident distributions.

\begin{table*}[t]
\centering
\caption{\textbf{Performance comparison in the RL stage.}
}
\label{tab:rl_results}
\setlength{\tabcolsep}{4pt}
\renewcommand{\arraystretch}{1.1}
\resizebox{\textwidth}{!}{%
\begin{tabular}{lcccccc>{\columncolor{orange!15}}c|ccc>{\columncolor{blue!15}}c}
\toprule
\multirow{2}{*}{\textbf{Model}} & \multicolumn{7}{c}{\textbf{In-Distribution Benchmarks}} & \multicolumn{4}{c}{\textbf{Out-of-Distribution Benchmarks}} \\
\cmidrule(lr){2-8} \cmidrule(lr){9-12}
 & \textbf{AIME25} & \textbf{AIME24} & \textbf{AMC23} & \textbf{MATH.} & \textbf{Miner.} & \textbf{Olymp.} & \textbf{Avg.} & \textbf{GPQA} & \textbf{MMLU.} & \textbf{ARC-C} & \textbf{Avg.} \\
\midrule
\multicolumn{12}{c}{\textbf{Vanilla GRPO}} \\
Vanilla GRPO  & 18.8  & 20.8  & 62.8  & 84.7  & 46.7  & 50.0  & 47.3  & 40.3  & 50.1  & 84.1  & 58.2 \\
\midrule
\multicolumn{12}{c}{\textbf{Hybrid SFT with RL}} \\
LUFFY~\citep{luffy}  & 29.4  & 23.1  & 65.6  & 87.6  & 49.5  & 57.2  & 52.1  & 39.9  & 53.0  & 80.5  & 57.8 \\
Prefix-RFT~\citep{prefix-rft}  & 26.4  & 31.8  & 68.2  & 88.4  & 50.9  & 55.7  & 53.6  & 39.1  & 52.1  & 84.0  & 58.4 \\
RL-PLUS~\citep{rl_plus}  & 25.9  & 33.4  & 68.1  & 90.2  & 52.3  & 58.8  & 54.8  & 40.4  & 54.7  & 82.3  & 59.1 \\
\midrule
\multicolumn{12}{c}{\textbf{SFT-then-RL Paradigm}} \\
SFT + RL  & 24.6  & 32.1  & 68.1  & 88.2  & 51.9  & 57.1  & 53.7  & 41.1  & 53.3  & 83.5  & 59.3 \\
GEM~\citep{gem} + RL  & 27.5  & 34.6  & 71.1  & 90.8  & 52.1  & 61.3  & 56.2  & 40.3  & 54.1  & 83.7  & 59.4 \\
DFT~\citep{dft} + RL  & 24.2  & 31.3  & 69.8  & 91.3  & 50.5  & 59.0  & 54.4  & 40.4  & 54.0  & 84.9  & 59.8 \\
PSFT~\citep{psft} + RL  & 26.7  & 36.7  & 71.5  & 91.2  & 52.0  & 62.7  & 56.8  & 42.8  & 55.4  & 85.1  & 61.1 \\
\midrule
\oursolution + RL & \textbf{30.4} & \textbf{39.2} & \textbf{72.7} & \textbf{91.7} & \textbf{54.9}  & \textbf{63.2}  & \textbf{58.7} & \textbf{43.2} & \textbf{56.0} & \textbf{85.9} & \textbf{61.7} \\
\textit{Impr. vs SFT+RL} & \textcolor{darkgreen}{\textbf{+5.8}}  & \textcolor{darkgreen}{\textbf{+7.1}}  & \textcolor{darkgreen}{\textbf{+4.6}}  & \textcolor{darkgreen}{\textbf{+3.5}} & \textcolor{darkgreen}{\textbf{+3.0}} & \textcolor{darkgreen}{\textbf{+6.1}}  & \textcolor{darkgreen}{\textbf{+5.0}}& \textcolor{darkgreen}{\textbf{+2.1}}  & \textcolor{darkgreen}{\textbf{+2.7}} & \textcolor{darkgreen}{\textbf{+2.4}}  & \textcolor{darkgreen}{\textbf{+2.4}} \\
\bottomrule
\end{tabular}%
}
\vspace{-3mm}
\end{table*}

\paragraph{Results.}
Table~\ref{tab:sft_results} shows that \oursolution~achieves the best overall performance on both in-distribution and OOD benchmarks. Compared to vanilla SFT, \oursolution~improves the ID average from 48.5\% to 51.0\% (\textbf{+2.5} points) and the OOD average from 51.3\% to 54.2\% (\textbf{+2.9} points), while preserving higher token entropy (0.31 $\rightarrow$ 0.43, \textbf{+0.12} nats).  
Enforcing a KL constraint slows the training while offering little benefit for entropy preservation. Among SFT variants, GEM achieves the highest entropy by keeping away from overconfident distributions. However, this occurs at the cost of ungrounded entropy, leading it to underperform \oursolution~on in-distribution tasks. PSFT and DFT improve training stability, but since they do not explicitly target exploration preservation, their overall improvements remain limited. As shown in Figure~\ref{fig:n-gram}, we also report N-gram diversity (1 minus N-gram similarity) and confirm that entropy preservation consistently increases generation diversity. Overall, these results indicate that \oursolution~ successfully preserves effective entropy while mitigating knowledge forgetting during the SFT stage.

\begin{figure}[t]
  \includegraphics[width=\columnwidth]{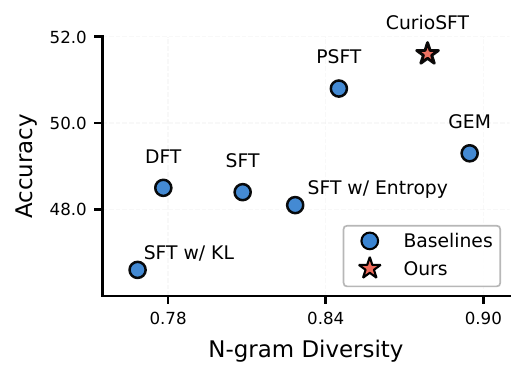}
  \vspace{-8mm}
  \caption{Accuracy vs. N-gram diversity. }
  \label{fig:n-gram}
  \vspace{-6mm}
\end{figure}

\paragraph{Ablation Study.}
We ablate key components of \oursolution~to quantify their contributions. Removing the Entropy-Guided Temperature Selection module (Section~\ref{sec:egat}) and using a fixed temperature $\hat{\tau}=1.3$ leads to a clear OOD drop (54.2\% $\rightarrow$ 52.2\%, \textbf{-2.0} points), highlighting the importance of token-level adaptivity for retaining knowledge. Next, we remove the separate teacher model and directly use the current model as the teacher, which causes a modest performance degradation, supporting the role of a stable teacher signal in providing reliable entropy-preserving guidance.

\paragraph{Complexity.}
\oursolution~introduces additional computation due to an extra forward pass and token-wise temperature search, but the overhead remains within a practical and acceptable range (43.2 $\rightarrow$ 53.9 seconds per step). Moreover, most of the cost can be reduced by removing the separate teacher, at the risk of a small performance drop.

\subsection{Unlocking the Potential of the SFT-then-RL Paradigm}
\label{sec:rl_results}

\paragraph{Baselines.}
We next examine whether the preserved entropy during SFT leads to \emph{meaningful} gains in the RL stage. We compare our solution against two families of baselines. First, we consider \textit{single-stage hybrid SFT+RL} methods that fuse offline demonstrations with on-policy exploration into one single stage. Specifically, LUFFY~\citep{luffy} optimizes an RL objective on a mixture of online rollouts and offline demonstrations; Prefix-RFT~\citep{prefix-rft} injects expert prefixes to steer exploration; and RL-PLUS~\citep{rl_plus} reuses expert examples during RL through multiple importance sampling. Second, we consider the standard \textit{two-stage SFT-then-RL} paradigm, where we run GRPO from the SFT checkpoints in Section~\ref{sec:sft_results}.

\paragraph{Results.}
Table~\ref{tab:rl_results} summarizes the results after the RL stage. \oursolution\,+\,GRPO pipeline achieves the best overall performance, improving the two-stage baseline SFT+RL from 53.7\% to 58.7\% on the in-distribution tasks (\textbf{+5.0} points) and from 59.3\% to 61.7\% on the OOD tasks (\textbf{+2.4} points). The gains are most pronounced on the challenging AIME benchmarks, where \oursolution+GRPO reach 39.2\% on AIME24 (vs.\ 32.1\% for SFT+RL) indicating that \oursolution~provides a substantially better initialization for RL exploration. Moreover, \oursolution\,+\,GRPO consistently outperforms single-stage hybrid methods, demonstrating that a well-designed SFT stage that preserves \emph{effective} exploration can unlock a higher RL performance ceiling.

\begin{table}[t]
\centering
\caption{Robustness across different backbones.}
\vspace{-2mm}
\label{tab:robustness}
\setlength{\tabcolsep}{4pt}
\renewcommand{\arraystretch}{1.3}
\resizebox{\linewidth}{!}{%
\begin{tabular}{lccccc>{\columncolor{orange!10}}c}
\toprule
\textbf{Method} & \textbf{AIME24/25} & \textbf{AMC} & \textbf{MATH.} & \textbf{Miner.} & \textbf{Olymp.} & \textbf{Avg.} \\
\midrule
\multicolumn{7}{c}{\textbf{Base model: Qwen3-4B-Base}} \\
Base model          &  6.3 / 4.2   & 42.0 & 53.0 & 17.9 & 20.8 & 24.0 \\
SFT            & 20.4 / 19.6  & 53.0 & 83.5 & 47.4 & 47.7 & 45.3 \\
SFT + GRPO  & 27.5 / 22.9  & 62.9 & 89.0 & 50.2 & 57.9 & 51.7 \\
\oursolution            & 21.3 / 25.0  & 54.4 & 84.2 & 48.3 & 48.2 & 46.9 \\
\oursolution + GRPO         & \textbf{28.8 / 27.9} & \textbf{65.0} & \textbf{89.7} & \textbf{53.6} & \textbf{59.1} & \textbf{54.0} \\
\midrule
\multicolumn{7}{c}{\textbf{Base model: Llama-3.1-8B-Instruct}} \\
Base model               & 2.1 / 2.5     & 19.3 & 43.2 & 26.3 & 14.8 & 18.0 \\
SFT                  & 8.3 / 11.7    & 38.3 & 68.1 & 32.7 & 35.6 & 32.5 \\
SFT + GRPO                 & 9.6 / \textbf{12.1}    & 40.3 & 74.1 & 35.5 & 38.9 & 35.1 \\
\oursolution         & 10.0 / 10.4    & 38.9 & 69.0 & 33.0 & 36.6 & 33.0 \\
\oursolution + GRPO  & \textbf{14.2} / 11.7 & \textbf{42.9} & \textbf{74.3} & \textbf{38.6} & \textbf{40.9} & \textbf{37.1} \\
\bottomrule

\end{tabular}%
}
\vspace{-5mm}
\end{table}
\subsection{Algorithm Robustness}
\label{sec:robustness}
We further evaluate \oursolution~on Qwen3-4B-Base~\citep{qwen3} and Llama-3.1-8B-Instruct~\citep{llama3}. As shown in Table~\ref{tab:robustness}, \oursolution~consistently improves over vanilla SFT on both backbones, improving the averaged accuracy from 45.3\% to 46.9\% on Qwen3-4B (\textbf{+1.6} points) and from 32.5\% to 33.0\% on Llama-3.1-8B-Instruct (\textbf{+0.5} points). These results validate that \oursolution~generalizes across model families and sizes.

\subsection{Hyperparameter Tuning}
\begin{figure}[t]
  \includegraphics[width=\columnwidth]{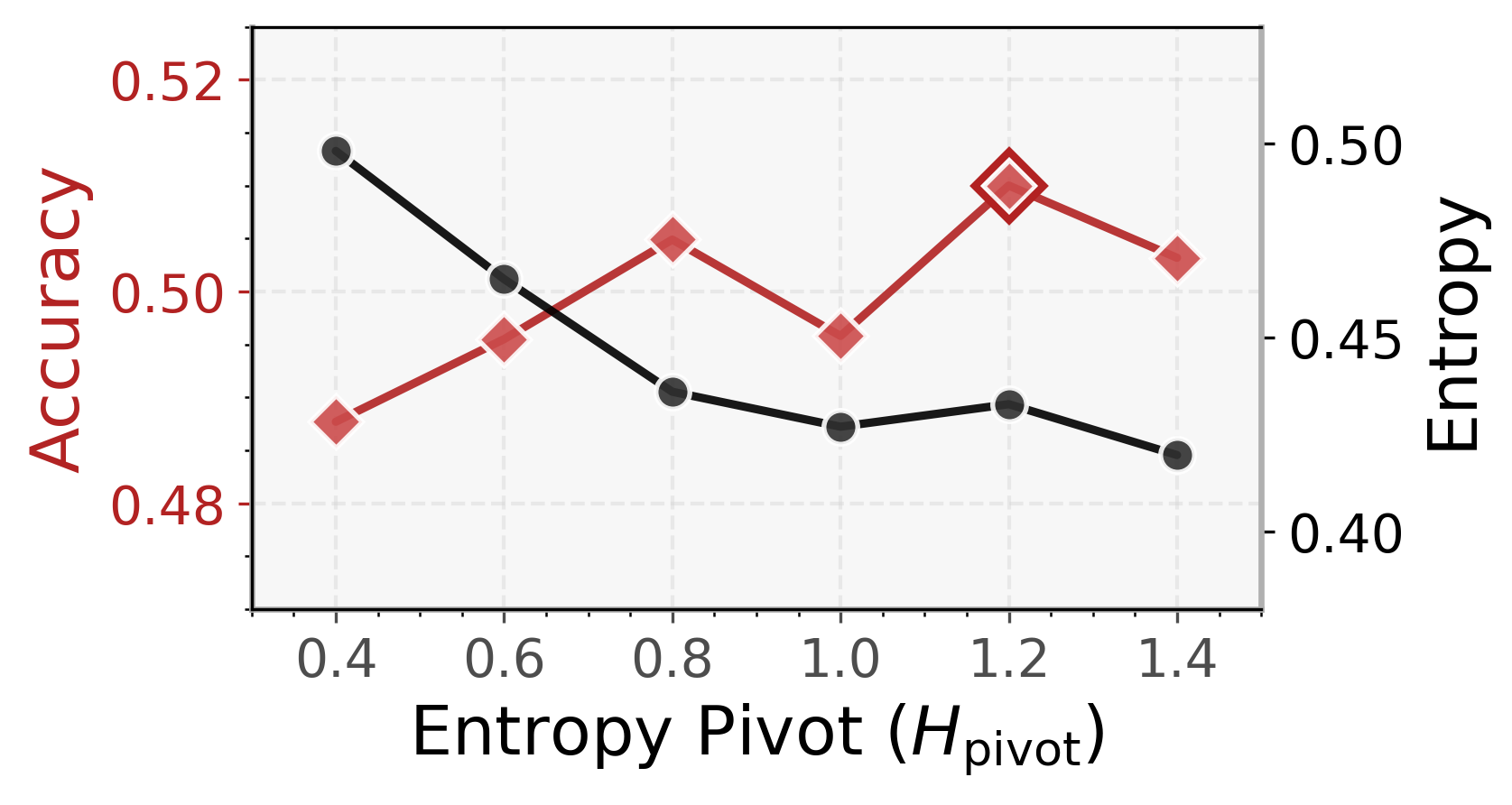}
  \vspace{-7mm}
  \caption{Sensitivity to the entropy pivot $H_{\text{pivot}}$.}
  \label{fig:hyper_tuning}
  \vspace{-5mm}
\end{figure}
Unlike entropy loss, which typically requires careful tuning of the loss weight, \oursolution~adaptively computes a token-wise temperature and thus avoids introducing an additional sensitive coefficient. In practice, the key hyperparameter is the entropy pivot $H_{\text{pivot}}$ in Eq.~(\ref{eq:entropy_increment}), which controls the overall strength of the entropy regularization. Increasing $H_{\text{pivot}}$ weakens the overall encouragement (fewer tokens receive a large entropy increment), while decreasing it makes the entropy increment larger for more tokens, thereby preserving more entropy. We sweep $H_{\text{pivot}}$ from 0.4 to 1.4. As shown in Figure~\ref{fig:hyper_tuning}, performance peaks around $H_{\text{pivot}}=1.2$. Importantly, across the entire range, \oursolution~consistently maintains performance gains over the baseline, indicating that the proposed method is robust to the choice of $H_{\text{pivot}}$.

\section{Related Work}

\paragraph{SFT in Post-Training.}
SFT plays an irreplaceable role in enhancing model capabilities during post-training~\citep{gpt,gemini}. Generally, SFT serves two primary functions: (a) serving as a large-scale \textit{knowledge injection}~\citep{sft_inject}, which significantly enhances zero-shot performance on OOD tasks~\citep{flan}; and (b) serving as a \textit{cold-start initialization} that enables the model to rapidly adapt to specific response patterns, thereby raising the performance ceiling for subsequent post-training stages (typically RL). For example, DeepSeek-R1~\citep{dpskr1} utilizes a small set of high-quality reasoning data to activate the model's inherent chain-of-thought capabilities before RL. ReTool~\citep{retool} constructs a diverse, high-quality tool-use dataset to instruct the model on when to invoke specific tools. In this paper, we focus on the latter role, investigating how SFT can be optimized to provide a \textit{superior initialization point} for the subsequent RL stage.

\paragraph{SFT-then-RL vs. Hybrid SFT with RL.}
Several studies argue that the two-stage ``SFT-then-RL'' paradigm may underperform compared to applying RL directly to the base model~\citep{sft_bad,sft_bad_extra,luffy,hpt,srft}. Hence, a line of research has explored fusing SFT and RL into a single-stage \textbf{hybrid paradigm}. For example, LUFFY~\citep{luffy} updates the model using both expert demonstrations and self-exploration rollouts via a weighted RL loss. RL-PLUS~\citep{rl_plus} introduces an exploration-based advantage function to balance SFT and RL losses. HPT~\citep{hpt} integrates the two objectives through a unified theoretical perspective. In contrast to these works, we empirically demonstrate that when intrinsic exploration capabilities are preserved during the SFT phase, the SFT-then-RL paradigm achieves a higher performance ceiling than hybrid approaches.

\paragraph{Diversity Regularization in SFT.}
The central challenge in SFT lies in its inherent susceptibility to overconfidence and diversity collapse. To address this, several works have explored regularization techniques. DFT~\citep{dft} re-weights token updates based on generation probabilities; GEM~\citep{gem} employs reverse KL divergence to prevent the distribution from converging to a collapsed mode; and PSFT~\citep{psft} imposes trust-region constraints to limit policy drift. However, none of these approaches address the problem from the perspective of entropy collapse, which is the key for effective exploration in the ``SFT-then-RL'' paradigm.

\section{Conclusion}
In this paper, we address a critical bottleneck in the SFT-then-RL paradigm: standard SFT induces entropy collapse that severely constricts downstream exploration. We propose \textbf{\oursolution}, an entropy-preserving SFT method utilizing adaptive self-distillation to maintain diverse yet valid exploration spaces. Experiments demonstrate that \oursolution~consistently outperforms vanilla SFT in both in-distribution and out-of-distribution tasks. More importantly, we verify that the preserved exploration effectively transfers to the RL stage, unlocking a significantly higher performance ceiling. 

\newpage
\section*{Limitations}
While \oursolution~effectively preserves exploration during SFT, we acknowledge two limitations. 
First, our method incurs additional training overhead compared to vanilla SFT, as it requires an extra forward pass to compute the teacher distribution and token-wise temperature selection. Nevertheless, the added cost remains within a practical range in our experiments. 
Second, our approach is bounded by the base model's intrinsic capabilities. Since we rely on self-distillation to preserve and amplify the model's latent exploration, the benefit may be smaller when the base model lacks sufficient prior knowledge. Future work could incorporate external signals to further enhance the exploration capability beyond what self-distillation alone can provide.

\bibliography{custom}

\newpage
\appendix

\appendix

\section{Derivation of Entropy Loss}
\label{apd:entropy_derivation}
Let $\mathcal{U}$ be the uniform distribution, i.e., $\mathcal{U}(y)=\frac{1}{|\mathcal{V}|}$.
The KL divergence from $\pi_\theta(\cdot\mid\mathbf{s})$ to $\mathcal{U}$ is:
\begin{align*}
D_{\mathrm{KL}}\!&\Big(\pi_\theta(\cdot\mid\mathbf{s}) \,\Big\|\, \mathcal{U}\Big) \\
&= \sum_{y\in\mathcal{V}} \pi_\theta(y\mid\mathbf{s})\log\frac{\pi_\theta(y\mid\mathbf{s})}{\mathcal{U}(y)} \\
&= \sum_{y\in\mathcal{V}} \pi_\theta(y\mid\mathbf{s})\Big(\log \pi_\theta(y\mid\mathbf{s}) - \log \tfrac{1}{|\mathcal{V}|}\Big) \\
&= \sum_{y\in\mathcal{V}} \pi_\theta(y\mid\mathbf{s})\log \pi_\theta(y\mid\mathbf{s}) + \log|\mathcal{V}| \\
&= -H\!\big(\pi_\theta(\cdot\mid\mathbf{s})\big) + \log|\mathcal{V}|.
\end{align*}
Since $\log|\mathcal{V}|$ is constant w.r.t.\ $\theta$, minimizing
$D_{\mathrm{KL}}\!\big(\pi_\theta(\cdot\mid\mathbf{s}) \,\|\, \mathcal{U}\big)$
is equivalent to maximizing $H\!\big(\pi_\theta(\cdot\mid\mathbf{s})\big)$.
Therefore, naive entropy regularization implicitly encourages the model distribution
to move toward the uniform distribution over $\mathcal{V}$.

\section{Proof of Monotonicity}
\label{app:entropy_proof}
\begin{theorem}
Let $\pi_\theta(\cdot \mid \mathbf{s};\tau)=\mathrm{softmax}\!\big(z_\theta(\cdot\mid\mathbf{s})/\tau\big)$
be the distribution derived from logits $z_\theta(\cdot\mid\mathbf{s})$ with temperature $\tau$.
Then the entropy $H\!\big(\pi_\theta(\cdot \mid \mathbf{s};\tau)\big)$ is non-decreasing with respect to $\tau$.
\end{theorem}
\begin{proof}
For brevity, let $\pi_y$ denote $\pi_\theta(y \mid \mathbf{s};\tau)$,
and let $z_y$ denote the logit for token $y$ (i.e., $z_y := z_\theta(y\mid\mathbf{s})$).
Recall that:
\[
\log \pi_y = \frac{z_y}{\tau} - \log Z(\tau),
\]
where $Z(\tau)$ is the partition function. The derivative of the entropy
$H(\pi) = -\mathbb{E}_{\pi}[\log \pi]$ with respect to $\tau$ is derived as follows:
\begin{align}
\frac{\partial H}{\partial \tau}
&= - \sum_{y} \frac{\partial \pi_y}{\partial \tau} \log \pi_y
    - \sum_{y} \pi_y \frac{\partial \log \pi_y}{\partial \tau} \nonumber \\
&= - \sum_{y} \frac{\partial \pi_y}{\partial \tau} \log \pi_y
\quad (\text{since } \sum_y \pi_y = 1) \nonumber \\
&= - \sum_{y} \frac{\partial \pi_y}{\partial \tau}
\left( \frac{z_y}{\tau} - \log Z \right) \nonumber \\
&= - \frac{1}{\tau} \sum_{y} z_y \frac{\partial \pi_y}{\partial \tau}.
\label{eq:short_deriv}
\end{align}
Using the standard derivative of the softmax function $\frac{\partial \pi_y}{\partial \tau}
= \frac{\pi_y}{\tau^2}\Big(\mathbb{E}_{\pi}[z] - z_y\Big)$
and substituting this into Eq.~\eqref{eq:short_deriv}, we have:
\begin{align}
\frac{\partial H}{\partial \tau}
&= - \frac{1}{\tau^3} \sum_{y} \pi_y z_y \Big(\mathbb{E}_{\pi}[z] - z_y\Big) \nonumber \\
&= \frac{1}{\tau^3} \left[ \sum_{y} \pi_y z_y^2 - \left(\sum_{y} \pi_y z_y\right)^2 \right] \nonumber \\
&= \frac{1}{\tau^3} \left( \mathbb{E}_{\pi}[z^2] - (\mathbb{E}_{\pi}[z])^2 \right) \nonumber \\
&= \frac{1}{\tau^3} \mathrm{Var}_{\pi}[z].
\end{align}
Since $\mathrm{Var}_{\pi}[z] \ge 0$ and $\tau > 0$, the derivative is always non-negative. Notably, the variance $\mathrm{Var}_{\pi}[z]$ vanishes if and only if the distribution $\pi$ is uniform. Thus, $H\!\big(\pi_\theta(\cdot \mid \mathbf{s};\tau)\big)$ is \textit{strictly non-decreasing} in $\tau$, except in the case of a uniform distribution.
\end{proof}

\section{Temperature scaling as the KL-closest higher-entropy teacher}
\label{sec:t_scaling}
To reduce overconfidence while staying close to $\pi$, we aim to construct a teacher
distribution $\pi^{\text{tch}}$ that satisfies two requirements:
(i) it has \emph{higher entropy} than the current policy, so it encourages exploration; and
(ii) it remains \emph{KL-close} to $\pi$, so the supervision signal is stable and capability-aware.
This naturally leads to the following constrained optimization problem:
\begin{align}
\label{eq:entropy_proj}
\mathbf{P1:}\,\,\min_{\pi^{\text{tch}}}\quad
& D_{\mathrm{KL}}\!\Big(\pi^{\text{tch}}(\cdot\mid\mathbf{s})
\,\Big\|\,
\pi(\cdot\mid\mathbf{s})\Big) \\
\text{s.t.}\quad
& H\!\big(\pi^{\text{tch}}(\cdot\mid\mathbf{s})\big)
\;\ge\;
H\!\big(\pi(\cdot\mid\mathbf{s})\big)+\Delta. \nonumber
\end{align}

\begin{theorem}[Temperature scaling is the optimum of \textbf{P1}]
\label{therem:t_scaling}
Given an entropy increment $\Delta$, the unique optimum $\pi_*^{\text{tch}}$ of \eqref{eq:entropy_proj} is a temperature-scaled distribution:
there exists a unique $\hat{\tau}>1$ such that
\begin{equation}
\label{eq:q_star_temp}
\begin{aligned}
\pi_*^{\text{tch}}(y\mid\mathbf{s})
&=\pi(y\mid\mathbf{s};\hat{\tau}) \\
&=\frac{\exp\!\big(z(y\mid\mathbf{s})/\hat{\tau}\big)}
{\sum_{y'\in\mathcal{V}}
\exp\!\big(z(y'\mid\mathbf{s})/\hat{\tau}\big)},
\end{aligned}
\end{equation}
and it satisfies
$H\!\big(\pi_*^{\text{tch}}(\cdot\mid\mathbf{s})\big)=H\!\big(\pi(\cdot\mid\mathbf{s})\big)+\Delta.$
\end{theorem}

\begin{proof}
We solve \textbf{P1} by forming its Lagrangian and applying the KKT conditions.
For convenience, denote the target entropy as
$H^{\text{tar}}=H\!\big(\pi(\cdot\mid\mathbf{s})\big)+\Delta.$
Rewrite the inequality constraint in the standard KKT form:
\[
g(\pi^{\text{tch}}):=H^{\text{tar}}-H\!\big(\pi^{\text{tch}}(\cdot\mid\mathbf{s})\big)\le 0,
\]
with KKT multiplier $\lambda_H\ge 0$. Since
$H(\pi^{\text{tch}})=-\sum_{y\in\mathcal{V}}\pi^{\text{tch}}(y\mid\mathbf{s})\log \pi^{\text{tch}}(y\mid\mathbf{s})$,
we have
$g(\pi^{\text{tch}})=H^{\text{tar}}+\sum_{y\in\mathcal{V}}\pi^{\text{tch}}(y\mid\mathbf{s})\log \pi^{\text{tch}}(y\mid\mathbf{s})$.
The Lagrangian of \textbf{P1} is:
\begin{equation}
\label{eq:lagrangian}
\begin{aligned}
\mathcal{L}&(\pi^{\text{tch}},\lambda,\lambda_H)=\\
&\sum_{y\in\mathcal{V}}
\pi^{\text{tch}}(y\mid\mathbf{s})
\log\frac{\pi^{\text{tch}}(y\mid\mathbf{s})}{\pi(y\mid\mathbf{s})} \\
&+\lambda\Big(\sum_{y\in\mathcal{V}}\pi^{\text{tch}}(y\mid\mathbf{s})-1\Big) \\
&+\lambda_H\Big(
H^{\text{tar}}
+\sum_{y\in\mathcal{V}}
\pi^{\text{tch}}(y\mid\mathbf{s})
\log \pi^{\text{tch}}(y\mid\mathbf{s})
\Big),
\end{aligned}
\end{equation}
where $\lambda$ is the multiplier for normalization and $\lambda_H\ge 0$ is the KKT multiplier for $g(\pi^{\text{tch}})\le 0$.
Taking the derivative w.r.t.\ $\pi^{\text{tch}}(y\mid\mathbf{s})$ and setting it to zero:
\begin{equation}
\label{eq:kkt_step}
\begin{aligned}
&(1+\log \pi^{\text{tch}}(y\mid\mathbf{s}))
-\log \pi(y\mid\mathbf{s})
+\lambda \\
&\qquad+\lambda_H\big(1+\log \pi^{\text{tch}}(y\mid\mathbf{s})\big)=0.
\end{aligned}
\end{equation}
Simplifying the above equation gives:
\begin{equation}
\label{eq:q_prop_p}
\log \pi^{\text{tch}}(y\mid\mathbf{s})
=
\frac{1}{1+\lambda_H}\log \pi(y\mid\mathbf{s}) + C,
\end{equation}
where $C$ is a normalization constant independent of $y$.
Define $\hat{\tau}:=1+\lambda_H$, then we obtain a power-form solution:
\begin{equation}
\label{eq:power_form}
\pi^{\text{tch}}(y\mid\mathbf{s})
\propto
\pi(y\mid\mathbf{s})^{1/\hat{\tau}}.
\end{equation}
When $\Delta>0$, the entropy constraint must be active at optimum; otherwise one could move $\pi^{\text{tch}}$ closer to $\pi$
and strictly decrease $D_{\mathrm{KL}}(\pi^{\text{tch}}\|\pi)$ while remaining feasible.
Thus, by complementary slackness, $\lambda_H>0$ and hence $\hat{\tau}>1$.
Using $\pi(y\mid\mathbf{s})\propto \exp(z(y\mid\mathbf{s}))$, we have:
\begin{equation}
\label{eq:power_to_temp}
\begin{aligned}
\pi^{\text{tch}}(y\mid\mathbf{s})
&\propto \pi(y\mid\mathbf{s})^{1/\hat{\tau}}
\;\propto\;
\Big(\exp(z(y\mid\mathbf{s}))\Big)^{1/\hat{\tau}} \\
&\propto\; \exp\!\big(z(y\mid\mathbf{s})/\hat{\tau}\big),
\end{aligned}
\end{equation}
which is exactly the temperature-scaled softmax form in \eqref{eq:q_star_temp} after normalization.
Finally, by Appendix~\ref{app:entropy_proof}, under the full-support assumption the entropy
$H\!\big(\pi(\cdot\mid\mathbf{s};\tau)\big)$ is continuous and strictly increasing in $\tau$,
so there exists a unique $\hat{\tau}>1$ such that
$H\!\big(\pi(\cdot\mid\mathbf{s};\hat{\tau})\big)=H^{\text{tar}}$. Since $D_{\mathrm{KL}}(\cdot\|\pi)$ is convex in its first argument and the feasible set is convex because entropy is concave, \textbf{P1} is a convex optimization problem. Therefore, any distribution that satisfies the KKT conditions is globally optimal. Consequently, the teacher distribution constructed by temperature scaling is the KL-closest distribution to $\pi$ among all distributions whose entropy is increased by at least $\Delta$.
\end{proof}

\section{Efficient Implementation of Entropy-Guided Temperature Search}
\label{app:binary_search}
Leveraging the monotonic relationship between entropy and the sampling temperature, we can find a unique temperature $\hat{\tau_t}$ whose entropy matches a desired target via binary search. To make this procedure efficient in large-scale LLM training, we implement it with two key design choices:

\begin{itemize}[leftmargin=*]
    \item \textbf{Vectorized binary search.}
    Instead of searching for $\hat{\tau_t}$ token by token, we perform a batched binary search over all tokens in a mini-batch using vectorized PyTorch operations.

    \item \textbf{Top-$k$ entropy approximation.}
    The softmax distribution over the vocabulary is typically heavy-tailed, so the entropy is dominated by the highest-probability tokens. We therefore approximate the full entropy using only the top-$k$ logits:
    \begin{equation}
        H\big(\pi_\theta(\cdot\mid\mathbf{s}_t)\big)
        \approx -\sum_{y \in \mathcal{V}_{:k}}
        \hat{\pi}_\theta(y\mid\mathbf{s}_t)
        \log \hat{\pi}_\theta(y\mid\mathbf{s}_t),
    \end{equation}
    where $\mathcal{V}_{:k}$ denotes the top-$k$ tokens and
    $\hat{\pi}_\theta$ is the distribution renormalized over this subset.
\end{itemize}
In all experiments, we set $k = 512$, which reduces the complexity of the entropy computation from $O(|\mathcal{V}|)$ to $O(k)$ and leads to only a modest training overhead.

\begin{table}[t]
\centering
\caption{Hyperparameters in SFT stage}
\label{appendix:sft_setting}
\begin{tabular}{l l}
\toprule
\textbf{Parameter Name}              & \textbf{Value}                                      \\ 
\midrule
Epochs & 3 \\
Batch Size     & 256  \\ 
Max Response Length   & 8192     \\ 
Learning Rate       & 1e-5           \\ 
Warm Up Style & cosine \\
Warm Up Steps & 60 \\
\bottomrule
\end{tabular}
\end{table}

\begin{table}[t]
\centering
\caption{Hyperparameters in RL stage}
\label{appendix:rl_setting}
\begin{tabular}{l l}
\toprule
\textbf{Parameter Name}              & \textbf{Value}                                      \\ 
\midrule
Training Steps & 500 \\
Batch Size     & 128  \\ 
Mini Batch Size    & 64      \\
Max Response Length   & 8192     \\ 
Learning Rate       & 1e-6           \\ 
Clip Higher  & 0.28 \\
Clip Lower & 0.2 \\
KL coefficient & 0.0 \\
Rollout Numbers & 8 \\
Rollout Temperature & 1.0 \\
Rollout Top\_p & 0.95 \\

\bottomrule
\end{tabular}
\end{table}

\section{Experiment Setting}
\label{appendix-exp-setting}

\paragraph{Training.}
For both SFT and RL, we use \texttt{Verl}~\citep{verl} as the training framework and vLLM~\citep{vllm} as the inference engine. All experiments are conducted on $8\times$ NVIDIA H800 GPUs. We update the teacher model every $n=5$ steps with a decay $\mu=0.99$. The full SFT hyperparameters are provided in Table~\ref{appendix:sft_setting}. For the RL stage, we generate $N=8$ candidate responses per question to estimate the advantage. We use a binary reward: a response receives reward 1 if it matches the ground truth (verified by \texttt{Math-Verify}~\citep{math-verify}) and follows the required format, and 0 otherwise. The full RL hyperparameters are listed in Table~\ref{appendix:rl_setting}. We use the prompt template in Table 7 for both SFT and RL training.

\paragraph{Dataset Processing.}
We use OpenR1-Math~\citep{luffy} for both SFT and RL training. We observe that a large portion of the questions contain irrelevant or distracting information. To reduce such noise, we rewrite the questions using DeepSeek-V3~\cite{dpsk-v3} and keep the original ground-truth answers unchanged. Example rephrase prompts are shown in Table 8.

\vspace{3mm}
\begin{paperbox}{Table 7: Training Prompt}
\label{box:rephrase}
<|im\_start|>system \\
You are an exceptional mathematician. Your task is to solve mathematical questions through a systematic and thorough reasoning process. This involves careful analysis, exploration of possible approaches, verification of intermediate steps, critical reassessment, and iterative refinement of your reasoning process. \\ \\
Structure your response in two distinct sections: ``Thought'' and ``Solution''. \\
In the ``Thought'' section, present your detailed reasoning process in the following format:\\
<think>\\
{Your detailed reasoning, including brainstorming, logical deductions, verification, and refinement of ideas.}\\
</think>\\
This section must conclude with ``</think>'', and should reflect deep, reflective, and self-correcting thinking process.\\
In the ``Solution'' section, following the ``</think>'', concisely draw the final, logical, and accurate answer from your reasoning. \\  \\
Please output your final answer within \textbackslash\textbackslash boxed\{\}.<|im\_end|>\\
<|im\_start|>user\\
Here is the question:\textcolor{blue}{\{question\}}<|im\_end|>\\
<|im\_start|>assistant\\
<think>\\
\end{paperbox}

\begin{paperbox}{Table 8: Question Rephrase Prompt}
\label{box:training}
I will provide a post from a math-related forum that contains a math problem. Your task is to extract only the math problem statement and remove any irrelevant or noisy content (e.g., commentary, solutions, chat, metadata). Keep the original wording and question type intact, and present the extracted problem clearly and concisely.
\\
Remove any redundant context, personal commentary, anecdotes, or unrelated information. But make sure not to change the meaning of the problem and keep all necessary mathematical or technical details.\\\\
Here are a few examples.\\
Example 1:\\
Input:\\
What is the remainder of $8^6+7^7+6^8$ is divided by $5$?\\
no calculator, of course, paper isn't needed either, but sure.\\
\\
Output:\\
What is the remainder of $8^6+7^7+6^8$ when divided by 5?\\
\\
Example 2:\\
Input:\\
(20 points) Let $x, y$ be non-zero real numbers, and satisfy $\frac{x \sin \frac{\pi}{5} + y \cos \frac{\pi}{5}}{x \cos \frac{\pi}{5} - y \sin \frac{\pi}{5}} = \tan \frac{9 \pi}{20}$.(1) Find the value of $\frac{y}{x}$;\\
\\
Output:\\
Let $x, y$ be non-zero real numbers, and satisfy $\frac{x \sin \frac{\pi}{5} + y \cos \frac{\pi}{5}}{x \cos \frac{\pi}{5} - y \sin \frac{\pi}{5}} = \tan \frac{9 \pi}{20}$. Find the value of $\frac{y}{x}$.
\\
\\
Now, here is the text you need to extract the problem.\\
\\
Input:\\
\textcolor{blue}{\{question\}}\\
\\
Output:\\
\end{paperbox}

\end{document}